%% file: cluster-regression-cdc19-final.tex
\pdfoutput=1
\documentclass[letterpaper, 10 pt, conference]{ieeeconf}  

\IEEEoverridecommandlockouts                              
\usepackage{lmodern}
\usepackage[T1]{fontenc}
\usepackage{graphicx}
\usepackage{epsfig} 
\usepackage{fancyhdr}
\usepackage{mathptmx} 
\usepackage{times} 
\usepackage{amsmath} 
\usepackage{amssymb}  
\usepackage{mathrsfs}
\usepackage{hyperref}  
\usepackage{url}
\usepackage{enumerate}
\usepackage{theorem}
\usepackage{booktabs} 
\usepackage{xcolor}
\usepackage{epstopdf}
\usepackage{algorithm}
\usepackage{algorithmic}

\usepackage{subfig}

\input{mymacros_2cln}

\input{lpsymbols}

\usepackage[normalem]{ulem}
\newtheorem{mythm}{Theorem}

\newtheorem{myremark}{Remark}

\newtheorem{mycorollary}{Corollary}

\makeatletter

\newcommand{\Rmnum}[1]{\expandafter\@slowromancap\romannumeral #1@}
\makeatother

\title{\LARGE \bf Convergence of Parameter Estimates for Regularized Mixed Linear
  Regression Models~\authorrefmark{1} \thanks{* Research partially supported by the
    NSF under grants DMS-1664644, CNS-1645681, and IIS-1914792, by the ONR under MURI
    grant N00014-16-1-2832, by the NIH under grant 1UL1TR001430 to the Clinical \&
    Translational Science Institute at Boston University, by the Boston University
    Digital Health Initiative, and by the Boston University Center for Information
    and Systems Engineering.}}

\author{Taiyao Wang and
  Ioannis Ch. Paschalidis,~\IEEEmembership{Fellow,~IEEE}
\thanks{$^{1}$Division of Systems Engineering, Boston University,
  Boston, MA 02446, USA. 
	{\tt wty@bu.edu}}
\thanks{$^{2}$Dept. of Electrical and Computer Engineering,
	Division of Systems Engineering,
	and Dept. of Biomedical Engineering, Boston University, 8 St. Mary's St.,
	Boston, MA 02215, USA.
	{\tt yannisp@bu.edu}, \url{http://sites.bu.edu/paschalidis/}.}
}

\begin{document}

\maketitle
\thispagestyle{empty}
\pagestyle{empty}

\begin{abstract}
We consider {\em Mixed Linear Regression (MLR)}, where training data have been
generated from a mixture of distinct linear models (or clusters) and we seek to
identify the corresponding coefficient vectors. We introduce a {\em Mixed Integer
  Programming (MIP)} formulation for MLR subject to regularization constraints on the
coefficient vectors. We establish that as the number of training samples grows large,
the MIP solution converges to the true coefficient vectors in the absence of noise.
Subject to slightly stronger assumptions, we also establish that the MIP identifies
the clusters from which the training samples were generated. In the special case
where training data come from a single cluster, we establish that the corresponding
MIP yields a solution that converges to the true coefficient vector even when
training data are perturbed by (martingale difference) noise. We provide a
counterexample indicating that in the presence of noise, the MIP may fail to produce
the true coefficient vectors for more than one clusters. We also provide numerical
results testing the MIP solutions in synthetic examples with noise.
\end{abstract}


\section{Introduction} \label{sec:intro}

{\em Mixed Linear Regression (MLR)} \cite{yi2014alternating,zhong2016mixed} is also
known as mixtures of linear regressions \cite{chaganty2013spectral} or cluster-wise
linear regression \cite{park2017algorithms}.  It involves the identification of two
or more linear regression models from unlabeled samples generated from an unknown
mixture of these models. This can be seen as a joint clustering and regression
problem. The problem is related to the identification of hybrid and switched linear
systems \cite{paoletti2007identification,vidal2008recursive} and has many diverse
applications.  In this work, we focus on the fundamental problem of establishing
strong consistency of parameter estimates, i.e., establishing that the estimated
parameters converge to their true values as the number of the training samples grows.

MLR is typically solved by local search such as {\em Expectation Maximization (EM)}
or alternating minimization, where one alternates between clustering 
and regression. 
It has recently been shown that EM
converges to the true parameters if it starts from a small enough neighborhood around
them \cite{yi2014alternating,yi2015regularized,balakrishnan2017statistical}.

Much effort has focused on the case where training samples are not perturbed by
noise. \cite{yi2016solving} proposed a combination of tensor decomposition and
alternating minimization, showing that initialization by the tensor method allows
alternating minimization to converge to the global optimum at a linear rate with high
probability (w.h.p.).  \cite{zhong2016mixed} proposes a non-convex objective function
with a tensor method for initialization so that the initial coefficients are in the
neighborhood of the true coefficients w.h.p.  \cite{hand2018convex} proposes a
second-order cone program, showing that it recovers all mixture components in the
noiseless setting under conditions that include a well-separation assumption and a
balanced measurement assumption on the data.  \cite{li2018learning} considers data
generated from a mixture of Gaussians, showing global convergence of the proposed
algorithm with nearly optimal sample complexity.

Establishing convergence (w.h.p.) and consistency results for MLR in the presence of
noise is much harder.  \cite{chaganty2013spectral} develops a provably consistent
estimator for MLR that relies on a low-rank linear regression to recover a
symmetric tensor, which can be factorized into the parameters using a tensor power
method.  \cite{chen2014convex} provides a convex optimization formulation for MLR
with two components and upper bounds on the recovery errors under subgaussian noise
assumptions. \cite{yen2018mixlasso} studies a \emph{MixLasso} approach but
convergence results are limited to the objective function and not the solution.
\cite{yin2018learning} develops an algorithm based on ideas from sparse graph codes;
the convergence results however are asymptotic under Gaussian noise and for MLR with
only two components.

Recently, an increasing number of machine learning and statistics problems have been
tackled using MIP
methods~\cite{bertsimas2015or,bertsimas2016best,xu2016joint,ACC-2018-PIEEE,ACC-diab-bri-xu-2019}.
\cite{bertsimas2007classification} proposes an MIP formulation for MLR and a
pre-clustering heuristic approach to solve large-scale problems.
\cite{park2017algorithms} proposes a Mixed-Integer Quadratic Program (MIQP)
formulation for MLR.

At the same time, {\em regularization} in learning problems has become widespread,
following the early success of
LASSO~\cite{Tib96,chatterjee2013assumptionless}. Recent work has obtained
regularization as a consequence of solving a robust learning problem (see,
\cite{chen2018robust} and references therein).

In this paper, we introduce a general MIP formulation for MLR subject to norm-based
regularization constraints. We establish that optimal solutions of the MIP converge
almost surely (rather than w.h.p.) to the true parameters in the noiseless case as
the sample size increases. Subject to cluster separability assumptions, we also
establish that MIP solutions can identify the proper cluster for each given
sample. For the special case of a single cluster, we show that the MIP solution
converges to the true parameter vector in the presence of noise satisfying a
martingale difference
assumption~\cite{lai1982least,chow2012probability,chen2012identification}. For
multiple clusters in the presence of noise we provide a counterexample, suggesting
that one can not in general recover the true parameters.

Our convergence analysis leverages techniques for proving strong consistency of
least-squares estimates for linear regression, but under weaker assumptions. The
related literature is substantial. A breakthrough paper \cite{lai1982least}
established strong consistency of least-squares estimates for stochastic regression
models. Asymptotic properties and strong consistency of least-squares parameter
estimates have been studied in many areas including system identification and
adaptive control \cite{chen2012identification}, econometric theory
\cite{nielsen2005strong} and time series analysis~\cite{wei1987adaptive}.  

The paper is organized as follows. Sec.~\ref{sec:pred}, presents the formulation of
the problem. The main results are given in Sec.~\ref{sec:pres}. Sec.~\ref{sec:sim}
presents numerical results and Sec.~\ref{sec:con} draws conclusions.

\textbf{Notation:} All vectors are column vectors and denoted by bold lowercase
letters. Bold uppercase letters denote matrices. For economy of space, we write $\bx
= (x_1, \ldots, x_{\text{dim}(\bx)})$ to denote the column vector $\bx$, where
$\text{dim}(\bx)$ is the dimension of $\bx$.  Prime denotes transpose and 
$\|\bx\|_p=(\sum_{i=1}^{\text{dim}(\bx)} |x_i|^p)^{1/p}$ the $\ell_p$ norm, where
$p\geq 1$.  Unless otherwise specified, $\|\cdot\|$ denotes the $\ell_2$ norm and
$\|\cdot\|_0$ the $\ell_0$ counting ``norm.'' $\lambda_{min}(\cdot)$ and
$\lambda_{max}(\cdot)$ denote the minimum and maximum eigenvalue of a (symmetric)
matrix. We use $\emptyset$ to denote the empty set, $[n]$ for the set $\{1, \ldots,
n\}$, and $|S|$ for the cardinality of the set $S$. We write $f(n)=O(g(n))$ if there
exist positive numbers $n_0$ and $c$ such that $f(n) \leq cg(n), \forall n\geq n_0$.
We write $f(n)=\Omega(g(n))$ if there exist positive numbers $n_0$ and $c$ such that
$f(n) \geq cg(n), \forall n\geq n_0$.  We write $f(n)=\Theta(g(n))$ if both
$f(n)=O(g(n))$ and $f(n)=\Omega(g(n))$ hold. Finally, $\E$ and $\P$ denote
expectation and probability, respectively.

\section{Problem Formulation} \label{sec:pred}

Consider the MLR model where data $(\bx_i, y_i) \in \mathbb{R}^{d+1}$ are generated by 
\begin{equation*}
    y_i = \bx_i' \bbeta_k + \epsilon_{i}, 
    \text{ for some } k \in [K],
\end{equation*}
where $\{\bbeta_k,\ \forall k \in [K]\}$ are the ground truth coefficient
vectors. The problem is to estimate these parameters from data. 

Given a training dataset $\{(\bx_i, y_i),\ i \in [n]\}$, we formulate the problem as
the following MIP: 
\begin{align} \label{eq:MIP}
\min_{\bbeta_k, t_i, c_{ki}} & \frac{1}{n} \sum_{i \in [n]} t_i^p  \\ 
    \text{s.t. } 
    & t_i - (y_i - \bx_i' \bbeta_k)
    +M(1-c_{ki}) \geq 0,\; i \in [n], k \in [K],\notag \\
     & t_i + (y_i - \bx_i' \bbeta_k)
    +M(1-c_{ki}) \geq 0,\; i \in [n], k \in [K],\notag\\
    & \sum_{k \in [K]} c_{ki} = 1,\; i \in [n], \notag\\
    & c_{ki} \in \{0,1\},\; i \in [n], k \in [K],\notag\\
    & t_i \geq 0,\; i \in [n],\notag\\
    & \| \bbeta_k  \|_q \leq d_{k,q},\; k \in [K], \notag
\end{align}
where $M$ is a large constant (big-$M$). Notice that when $c_{ki}=1$, the first two
constraints imply $t_i \geq |(y_i - \bx_i' \bbeta_k)|$, whereas $c_{ki}=0$ implies
that no constraint is imposed on $t_i$ since $M$ is a large positive constant. At
optimality, (\ref{eq:MIP}) minimizes a $p$-norm loss function for the regression
problem and assigns each data point to the cluster achieving minimal loss. The last
constraint in (\ref{eq:MIP}) imposes a $q$-norm regularization constraint to each
coefficient vector $\bbeta_k$ and $d_{k,q}$ are given constants.  In statistics and
machine learning, regularization is widely used to help prevent models from
overfitting the training data~\cite{chen2018robust}.  A Bayesian understanding of
regularization is that regularized least squares are equivalent to priors on the
solution to the least squares problem. For $p,q=1$, (\ref{eq:MIP}) is a linear MIP,
while if either $p$ or $q$ (or both) are equal to $2$ it is a quadratic MIP. Both can
be solved by modern MIP solvers.

Without the regularization constraint, another equivalent formulation of (\ref{eq:MIP})
is 
\begin{equation} \label{eq:MIP_unconstrained}
    \min\limits_{\bbeta_k} \frac{1}{n} \sum_{i \in [n]}
    \min\limits_{k \in [K]}
    |y_i - \bx_i' \bbeta_k|^{p}.  
\end{equation}

Let $\{\bbeta^n_k ,t_i^n, c_{ki}^n,\ k\in[K],\ i\in [n]\}$ denote an optimal solution
to (\ref{eq:MIP}); we use a superscript $n$ to explicitly denote dependence on the
training set.  Define $\Omega_k^{n} = \{i \in [n]:\ y_i = \bx_i'\bbeta_k +
\epsilon_{i}\}$ and $\hat{\Omega}^{n}_k = \{i \in [n]:\ c^n_{ki} = 1\}$ which form
the true and the estimated partition of the training set into the $K$ clusters,
respectively. We can show the following lemma (proof omitted due to space
limitations).
\begin{lemma}
$\forall k \in [K],$
\begin{align*} 
\hat{\Omega}^{n}_k & \subset \{i \in [n]:\ |y_i - \bx_i' \bbeta^{n}_k | = \min_{m \in
  [K]} |y_i - \bx_i' \bbeta^{n}_m |   \},\\ 
\hat{\Omega}^{n}_k & \supset \{i \in [n]:\ |y_i - \bx_i' \bbeta^{n}_k | < 
\min_{m \neq k \in [K]}|y_i - \bx_i' \bbeta^{n}_m |\}.
\end{align*}
\end{lemma}

In order to analyze the strong consistency of parameter estimates obtained by the MIP
formulation, we introduce the following assumptions. 
\begin{itemize}
\item[(A1)] The clusters are not degenerate and different, i.e., $ |\Omega_{k}^n| =
  \Theta(n)$, $\forall k \in [K]$ and $\bbeta_k \neq \bbeta_{h}$, $\forall k \neq h$.

\item[(A2)]
 The noise sequence $\{\epsilon_{i}, {\mathscr F}_{i}\}$ is a martingale difference sequence, where $\{ {\mathscr F}_{i}\}$ is a sequence of increasing $\sigma$-fields, and
$\sup_{i}\E[\vert \epsilon_{i}\vert ^{2}\vert {\mathscr F}_{i-1}]<\infty,$  a.s.  

\item[(A3)]
 The noise sequence $\{\epsilon_{i}, {\mathscr F}_{i}\}$ is a martingale difference sequence, where $\{ {\mathscr F}_{i}\}$ is a sequence of increasing $\sigma$-fields, and
$\sup_{i}\E[\vert \epsilon_{i}\vert ^{\alpha}\vert {\mathscr F}_{i-1}]<\infty, \text{
   a.s. for some } \alpha>2$.   

\item[(A4)]
$ \| \bbeta_k  \|_q \leq d_{k,q},\ \forall k \in [K],\ \forall q \in \{2, 1, 0\}$.
\end{itemize}
 The noise assumptions are mild. For instance, (A2) holds if $\{\epsilon_i\}$ are
 $i.i.d.$ random variables with zero mean and variance, including but not limited to
 Gaussian white noise.

The following lemmata are important in proving the strong consistency of least squares
estimates in stochastic linear regression models.
\begin{lemma}(\cite{lai1982least})
Suppose the noise sequence $\{\epsilon_i\}$ satisfies Assumption (A2).  Let $\bx_i$
be ${\mathscr F}_{i-1}$ measurable for every $i$.  Define $N=\inf \{n:\ \sum_{i=1}^{n}
\bx_i\bx_i' \text{ is nonsingular }\}$. 
Assume that $N <\infty$ a.s., and for $n \ge N $, define
\begin{align*}
Q_n=\bigg( \sum_{i=1}^{n}   \bx_i \epsilon_i\bigg)'  
\bigg(\sum_{i=1}^{n}   \bx_i\bx_i' \bigg)^{-1} 
\bigg( \sum_{i=1}^{n}   \bx_i \epsilon_i \bigg).
\end{align*}
Let $\lambda_{max}(n)$ the maximum eigenvalue of $\sum_{i=1}^{n} \bx_i\bx_i'$. 
Then $ \lambda_{max}(n)$ is non-decreasing in $n$ and 
\begin{itemize}
\item[(i)] On $(\lim\limits_{n \to \infty} \log \lambda_{max}(n) <  \infty )$,
$Q_n=O(1)$ a.s. 

\item[(ii)]
 On
($ \lim\limits_{n \to \infty} \log \lambda_{max}(n) = \infty )$, we have that for 
 $ \forall \delta >0$,
\begin{align*}
Q_n= O( (\log \lambda_{max}(n) )^{1+ \delta})\; \text{a.s.}
\end{align*}

\item[(iii)]
 On $ ( \lim\limits_{n \to \infty} \log \lambda_{max}(n) = \infty )$, the previous
 result can be strengthened to 
\begin{align*}
Q_n=O(  \log \lambda_{max}(n) )  \textrm{ a.s.,} 
\end{align*}
if  the assumption (A2) is replaced by (A3).
\end{itemize}
\label{thm:laiwei1982lemma1Q_n}
\end{lemma}

The following estimates for the weighted sums of martingale difference sequences are
based on the Kronecker lemma, the local convergence theorem and the strong law for
martingales \cite{chow2012probability,chow1965local}.
\begin{lemma}(\cite{lai1982least,chen2012identification})
Suppose the noise sequence
$\{\epsilon_i \}$ satisfies the assumption (A2).
Let $u_i$ be ${\mathscr F}_{i-1}$ measurable for every $i$ and $s_n = (\sum_{j \in
  [n]} u_{j}^2)^{\frac{1}{2}}$. 
Then, 
\begin{itemize}
\item[(i)]
\begin{align*}
    \sum_{j \in [n]} u_{j} \epsilon_{j}
    \text{ converges a.s. on }
    s_n^2 < \infty.
\end{align*}
\item[(ii)]
\begin{align*}
    \sum_{j \in [n]} u_{j} \epsilon_{j} = 
o(s_n [log(s_n^2)]^{\frac{1}{2}+\delta})
\text{ a.s. }
 \forall \delta > 0.
\\
\text{ on }
\sum_{j \in [n]} u_{j}^2 = \infty.
\end{align*}

\item[(iii)] 
\begin{align*}
    \sum_{j \in [n]} u_{j} \epsilon_{j} = 
O( s_n [log(s_n^2)]^{\frac{1}{2}})
\text{ a.s. on }
\sum_{j \in [n]} u_{j}^2 = \infty,
\end{align*}
if  the assumption (A2) is replaced by (A3).

\end{itemize}
\label{thm:laiwei1982lemma2}
\end{lemma}

The estimates given by these results are not as sharp as those given by the law of
the iterative logarithm \cite{chow2012probability} but the assumptions needed here
are much more general.

\section{Main Results} \label{sec:pres}
\subsection{Noiseless case} 

Suppose $\{\bbeta^n_k,\ \forall k \in [K]\}$ are optimal solutions to (\ref{eq:MIP})
for $p \in \{2, 1\}$ and $q \in \{2, 1, 0\}$.
\begin{mythm}
\label{thm:Noiseless}
Suppose Assumptions (A1) and (A4) hold in the MLR model, and
\begin{equation}
\liminf_{|\scrS|\rightarrow\infty} \lambda_{min}\bigg(\sum_{i \in \scrS \subset \Omega_k^{n}}
\bx_{i}\bx_{i}'\bigg) > 0,\text{ a.s.},\ \forall k \in [K]. 
\label{eq:condition_x_nonoise}
\end{equation}
Then we have the strong consistency, i.e.,
\[ 
\exists N \text{ s.t. } \bbeta^{n}_k =  \bbeta_{\pi(k)},
\text{ a.s.},  \forall n>N, k \in [K],
\]
where $\pi$ is a permutation of the $K$ clusters.

Furthermore, if the optimal solution of (\ref{eq:MIP}) is unique, or there are no
ties for assigning samples to clusters, i.e., $\forall i \in [n], \forall k \neq h
\in [K]$, $|\bx_i' (\beta_k-\beta_h)| >0,$ it follows
\begin{align}
\label{eq:set_converge}
\exists N \text{ s.t. } 
\hat{\Omega}^{n}_{k} = \Omega_{\pi(k)}^{n},\text{ a.s., } \forall n>N. 
\end{align}
\end{mythm}

\begin{proof}
$\forall k \in [K],\ \exists m(k) \in [K]$ such that 
\[ 
|\Omega_k^{n} \cap  \hat{\Omega}^{n}_{m(k)}| \geq |\Omega_k^{n} \cap  \hat{\Omega}^{n}_{h}|,\; \forall h \in [K].
\]
Consequently, 
\[
|\Omega_k^{n} \cap  \hat{\Omega}^{n}_{m(k)}| \geq |\Omega_k^{n}|/K = \Theta(n)
\rightarrow\infty, \text{ when } n \rightarrow\infty.
\]

From (\ref{eq:condition_x_nonoise}), $\exists N,$ such that $|\Omega_k^{n} \cap
\hat{\Omega}^{n}_{m(k)}|>N,$ and 
\begin{align}
    \label{eq:condition_x_nonoise2}
    \lambda_{min}\bigg(\sum_{i \in \Omega_k^{n} \cap  \hat{\Omega}^{n}_{m(k)}}
    \bx_{i}\bx_{i}'\bigg) > 0, \text{ a.s.} 
\end{align}

Since the true $\{\bbeta_k\}$ are a feasible solution of (\ref{eq:MIP}) and there is
no noise, the optimal objective value must be $0$, i.e., $y_i = \bx_i' \bbeta_k$ for
some $k$ and all $i$. It follows
 \begin{align*}
    y_i = \bx_i' \bbeta_k = \bx_i' \bbeta^n_{m(k)}, \quad \forall i \in \Omega_k^{n}
    \cap  \hat{\Omega}^{n}_{m(k)}.  
\end{align*}
As a result, 
\begin{align*}
    \bx_i'(\bbeta_k - \bbeta^n_{m(k)}) = 0,
     \end{align*}
and
\begin{equation} \label{eq:condition_x_nonoise4}
    \sum_{i \in \Omega_k^{n} \cap  \hat{\Omega}^{n}_{m(k)}} |\bx_i'(\bbeta_k - \bbeta^n_{m(k)})|^2 = 0. 
 \end{equation}

From the definition of the smallest eigenvalue, we have
\begin{multline}     \label{eq:condition_x_nonoise3}
   \sum_{i \in \Omega_k^{n} \cap  \hat{\Omega}^{n}_{m(k)}} |\bx_i'(\bbeta_k -
   \bbeta^n_{m(k)})|^2 \\
    \geq
    \lambda_{min}\bigg(\sum_{i \in \Omega_k^{n} \cap  \hat{\Omega}^{n}_{m(k)}}
    \bx_{i}\bx_{i}'\bigg) \|\bbeta_k - \bbeta^n_{m(k)} \|^2.  
 \end{multline}
 
Accordingly, when $|\Omega_k^{n} \cap \hat{\Omega}^{n}_{m(k)}|>N,$ it follows from
equations (\ref{eq:condition_x_nonoise2}), (\ref{eq:condition_x_nonoise4}) and
(\ref{eq:condition_x_nonoise3}) that
$\bbeta_k - \bbeta^n_{m(k)} = 0$.

Next, we show the mapping $m(\cdot)$ is a bijection by contradiction.  Assume there
exists $m(k) = m(h)$ for $k \neq h$. Then, $\bbeta_k = \bbeta^n_{m(k)} =
\bbeta^n_{m(h)} = \bbeta_{h}$ when $n$ is large enough, which contradicts the
assumption that the clusters are different.  Thus, $\pi = m^{-1}$ is a permutation.
Finally, we can prove the cluster set equality (\ref{eq:set_converge}) by
contradiction.
\end{proof}

\begin{myremark}
\begin{enumerate}[(i)]
 \item Equation (\ref{eq:condition_x_nonoise}) on $\bx_i$ is not only
   sufficient but also necessary. 
  $\scrS$ in (\ref{eq:condition_x_nonoise}) is the subset of $\Omega_k^{n}$.
      For instance, if $\bx_i= (1,1),\ \forall i \in
   \Omega_k^{n}$, and $\lambda_{min}(\sum_{i \in \Omega_k^{n} } \bx_{i}\bx_{i}^{'}) =
   0,$ which violates Equation (\ref{eq:condition_x_nonoise}), and model parameters
   are not identifiable no matter which method is being used.
\item Thm.~\ref{thm:Noiseless} holds for either $p=1$ or $p=2$.  From a computational
  complexity aspect, a linear MIP ($p=1, q=1$) can be solved faster than a quadratic
  MIP ($p=2$).
\item Even if we have the strong consistency for the model parameters, i.e.,
  $\bbeta^{n}_k \to \bbeta_{\pi(k)},\text{ a.s.}, \forall k \in [K]$, we may not have
  $\hat{\Omega}^{n}_{k} \to \Omega_{\pi(k)}^{n},\text{ a.s.}$ when $n \to \infty$.
  For instance, if $\bx_i= \mathbf{0}$, for some $i$, then the sample $i$ may be
  assigned to any cluster.
\end{enumerate}
\end{myremark}

Thm.~\ref{thm:Noiseless} holds for any $p>0$ and $q\ge 0$.  Assumption
(\ref{eq:condition_x_nonoise}) is equivalent to requiring that every cluster has at
least $d$ linearly independent measurements.  The exact convergence can also be
achieved by other methods, e.g., the second-order cone program in
\cite{hand2018convex}. However, to the best of our knowledge, our assumptions are the
weakest since we do not need a ``sufficient-separation'' and a balanced measurement
assumption on the data as in \cite{hand2018convex}.

\subsection{Noisy case with a single cluster} 
Linear regularized regression can be seen as a specific case of MLR with a single
cluster.  In this section, we assume $K=1$ and consider the presence of noise.  We
establish strong consistency for the model parameters under general covariate and
noise distributions. Consider the regularized regression problem
\begin{align} \label{eq:LS}
\min_{\bbeta} & \frac{1}{n} \sum\limits_{i \in [n]} |y_i - \bx_i' \bbeta |^2  \\ 
\text{s.t. } 
& \| \bbeta  \|_q \leq d_{q},\ \quad \forall k \in [K], \notag
\end{align}
where $q \in \{2, 1, 0\}$, corresponding to ridge regression, LASSO and regression
with a subset selection constraint, respectively. For $q=0$, the problem can be
formulated as a MIP problem~\cite{bertsimas2016best}. Let $\bbeta^n$ denote an
optimal solution of (\ref{eq:LS}).

Let $\lambda_{max}(n)= \lambda_{max} (\sum_{i=1}^{n} \bx_i\bx_i')$ and $
\lambda_{min}(n)= \lambda_{min} (\sum_{i=1}^{n} \bx_i\bx_i') $.
\begin{mythm}
\label{thm:noise2order}
Suppose that Assumptions (A1), (A2) and (A4) hold for (\ref{eq:LS}), and let $\bx_i$
be ${\mathscr F}_{i-1}$ measurable for every $i$.  If 
\begin{align*}
    \lambda_{min}(n) \to \infty, \text{ a.s.,}
\end{align*}
then for all $\delta > 0$,
\[
\|\bbeta^{n} - \bbeta \|
\leq
o(\lambda^{\frac{1}{2}}_{max} (n)   [log \lambda_{max} (n)]^{\frac{1}{2}+\delta}
/\lambda_{min}(n)),  
\text{ a.s.} 
\]
\end{mythm}

\begin{proof}
Since $\bbeta$ is a feasible solution of (\ref{eq:LS}) (cf. Ass.~(A4)), we have
\[
\sum\limits_{i \in [n]} |y_i - \bx_i' \bbeta^{n} |^2 
\leq
\sum\limits_{i \in [n]} |y_i - \bx_i' \bbeta |^2
=
\sum\limits_{i \in [n]} |\epsilon_i |^2. 
\]
Substituting $y_i=\bx_i' \bbeta + \epsilon_i$ in the l.h.s., we obtain
\[
\sum\limits_{i \in [n]} |\epsilon_i- \bx_i' (\bbeta^{n} - \bbeta ) |^2
\leq
\sum\limits_{i \in [n]} |\epsilon_i |^2, 
\]
and by expanding the l.h.s. we obtain
\begin{equation}
\label{eq:proof1}
\sum\limits_{i \in [n]} | \bx_i' (\bbeta^{n} - \bbeta) |^2
\leq
2 \sum\limits_{i \in [n]}  \epsilon_i \bx_i' (\bbeta^{n} - \bbeta ).
\end{equation}
From the definition of the minimum eigenvalue, we have
\begin{equation}
\label{eq:proof2}
\sum\limits_{i \in [n]} | \bx_i' (\bbeta^{n} - \bbeta ) |^2 
\geq
\lambda_{min}(n) \|\bbeta^{n} - \bbeta \|^2.
\end{equation}

Next, we study the r.h.s. of (\ref{eq:proof1}) in order to bound the convergence rate
of $\|\bbeta^{n} - \bbeta \|$.  From Lemma \ref{thm:laiwei1982lemma2}(ii), we have
that for any $j \in [d]$,
\[
 \sum\limits_{i \in [n]}  \epsilon_i x_{ij} =
o(s_{nj} [log(s_{nj}^2)]^{\frac{1}{2}+\delta})\
\text{a.s., where}\ s_{nj} = \bigg(\sum_{i \in [n]} x_{ij}^2\bigg)^{\frac{1}{2}}.
\]
For any $j \in [d]$,
\begin{align*}
s_{nj} & = \bigg(\sum_{i \in [n]} x_{ij}^2\bigg)^{\frac{1}{2}}
\leq
\bigg(\sum_{i \in [n]} \sum_{j \in [d]} x_{ij}^2\bigg)^{\frac{1}{2}} \\
& \leq 
\left( \text{Tr}\bigg(\sum_{i=1}^{n} \bx_i\bx_i' \bigg) \right)^{\frac{1}{2}}
= \Theta\bigg( \lambda^{\frac{1}{2}}_{max}(n) \bigg),
\end{align*}
where $\text{Tr}(\cdot)$ denotes the trace of a matrix. 
By combining the above two equations, we have for any $j \in [d]$,
\[
 \sum\limits_{i \in [n]}  \epsilon_i x_{ij} =
o(\lambda^{\frac{1}{2}}_{max} (n)   [log \lambda_{max} (n)]^{\frac{1}{2}+\delta})
\text{ a.s., }
\]
and thus,
\[
\left\| \sum_{i \in [n]}  \epsilon_i \bx_i \right\|=
o(\lambda^{\frac{1}{2}}_{max} (n)   [log \lambda_{max} (n)]^{\frac{1}{2}+\delta}) \text{ a.s.}
\]
We bound the r.h.s. of (\ref{eq:proof1}) as
\begin{align} \label{pr3}
 2 \sum_{i \in [n]}  \epsilon_i \bx_i'(\bbeta^{n} - \bbeta )
& \leq 
2 \left\| \sum_{i \in [n]}  \epsilon_i \bx_i \right\|  \|\bbeta^{n} - \bbeta \|
\notag \\
&\leq
 o(\lambda^{\frac{1}{2}}_{max} (n)   [log \lambda_{max} (n)]^{\frac{1}{2}+\delta})
\|\bbeta^{n} - \bbeta \|.
\end{align}
Combining (\ref{eq:proof1}), (\ref{eq:proof2}) and (\ref{pr3}),
we obtain
\[
\lambda_{min}(n) \|\bbeta^{n} - \bbeta \|^2 
\leq
o(\lambda^{\frac{1}{2}}_{max} (n)   [log \lambda_{max} (n)]^{\frac{1}{2}+\delta})
\|\bbeta^{n} - \bbeta \|.
\]
\end{proof}

If the Assumption (A2) is replaced by (A3), we have a stronger version of the
previous theorem (proof omitted). 
\begin{mythm}
\label{thm:noise3order}
Suppose that Assumptions (A1), (A3) and (A4) hold for (\ref{eq:LS}), 
and let $\bx_i$ be ${\mathscr F}_{i-1}$ measurable for every $i$.
 If 
\[
\lambda_{min}(n) \to \infty, \text{ a.s.,}
\]
then we have 
\[
\|\bbeta^{n} - \bbeta \| 
\leq 
O(\lambda^{\frac{1}{2}}_{max} (n)   [log \lambda_{max} (n)]^{\frac{1}{2}}
/\lambda_{min}(n)  ) 
\text{ a.s.} 
\]
\end{mythm}

As a point of comparison, the classical convergence rate for least square estimates
of unregularized linear regression subject to martingale difference noise is given by:
\begin{align} \label{conv-classical}
\|\bbeta^{n} - \bbeta \| & =
o(  [log \lambda_{max} (n)]^{\frac{1}{2}+\delta}
/\lambda^{\frac{1}{2}}_{min}(n)  ) 
\text{ a.s. for } \alpha=2, \notag\\
& = 
O(  [log \lambda_{max} (n)]^{\frac{1}{2}}
/\lambda^{\frac{1}{2}}_{min}(n)  ) 
\text{ a.s. for }   \alpha>2.
\end{align}
The convergence rate in Theorem \ref{thm:noise2order} and Theorem
\ref{thm:noise3order} is slower than the unregularized linear regression case in
general. However, it can be seen that for well-conditioned data, i.e., $\lambda_{max}
(n) \sim \lambda_{min} (n)$, we have the same convergence rate with the unregularized
case. The following Corollary outlines sufficient conditions to that effect.

\begin{mycorollary}
Suppose $\{\bx_i\}$ are i.i.d. random vectors with $\E[\bx_i\bx_i']$ being a positive
definite matrix. Suppose $\{\epsilon_i\}$ are i.i.d. random variables, independent of
the $\bx_i$, with zero mean and variance $\sigma^2>0$. Then, we have
\[
\|\bbeta^{n} - \bbeta \|^2 
= 
o(n^{-\frac{1}{2}}   [log (n)]^{\frac{1}{2}+\delta} ),
\text{ a.s. }
 \forall \delta > 0.
\]
Furthermore, if $\E[|\epsilon_i|^\alpha] < \infty, \text{ a.s. for some } \alpha>2$,
we have
\begin{equation} \label{eq:coro1_3order}
\|\bbeta^{n} - \bbeta \|^2 
= 
O(n^{-\frac{1}{2}}   [log (n)]^{\frac{1}{2}} ),
\text{ a.s. }
\end{equation}
\end{mycorollary}

\begin{proof}
From the strong law of large numbers,
\[
  \lim_{n\rightarrow\infty} \frac{1}{n} (\sum_{i \in [n]}^{}   \bx_i\bx_i' )
    =
    \E[\bx_i\bx_i'],\; \text{ a.s. }
\]
It follows 
\[
     \lambda_{max}(n)  = \Theta(n), \qquad \lambda_{min}(n)  = \Theta(n).
\]
Thus, from Thm.~\ref{thm:noise2order} we have
\begin{align*}
\|\bbeta^{n} - \bbeta \|
\leq &
o(\lambda^{\frac{1}{2}}_{max} (n)   [log \lambda_{max} (n)]^{\frac{1}{2}+\delta}
/\lambda_{min}(n)  ) \\
= &
o(n^{-\frac{1}{2}}   [log (n)]^{\frac{1}{2}+\delta} ),
\text{ a.s. }
 \forall \delta > 0.
\end{align*}
Similarly, using Thm.~\ref{thm:noise3order} we can prove (\ref{eq:coro1_3order}).
\end{proof}

We point out that our setting is more general, including ridge regression, LASSO and
regression with subset selection.  The convergence rate is sharper in terms of the
sample size than the rate achieved for LASSO in \cite{chatterjee2013assumptionless}.
In addition, our noise assumptions are more general and weaker than the Gaussian
noise used in \cite{chatterjee2013assumptionless}.

\subsection{A counterexample for the noisy case with two clusters}

Next, we provide a counterexample indicating that the presence of noise may prevent
convergence to the true coefficients when we have more than a single cluster.

Consider the 2-cluster MLR model where the data $(x_i, y_i) \in \mathbb{R}^{2}$ are
generated as follows:
\begin{align*}
x_i =& 1, \\
\epsilon_{i} \in & \{ 1, -1 \} \text{ with probability } 0.5,\\
\beta_k \in & \{ \delta, 0 \} \text{ with probability } 0.5,\\
y_i =& x_i \beta_k + \sigma \epsilon_{i} \in \{ \sigma, -\sigma, \delta+\sigma, \delta-\sigma\}, 
\end{align*}
where $\{\epsilon_i\}$ are i.i.d. random variables with zero mean and positive
variance.  If $\delta < \sigma$, the optimal solution of (\ref{eq:MIP}) will be
different from the ground truth parameters. In particular, the optimal objective
function value of (\ref{eq:MIP}), is smaller than the objective function value of the
ground truth parameters, and depend on $\delta$ rather than $\sigma$. 

Not only (\ref{eq:MIP}), but also other methods may fail to recover the ground truth
parameters since they are ``unidentifiable,'' i.e., two sets of parameters can
generate the same distribution.  This counterexample shows strong consistency may
fail to hold under the weakest assumptions used in our earlier analysis.

\section{Numerical results} \label{sec:sim}
In this section we provide numerical results on the convergence of parameters
estimates in MLR obtained by formulation (\ref{eq:MIP}). Computations were performed
with GUROBI 8.0 and python 3.6.5 with 16 CPUs on the Boston University Shared Computing Cluster.
The results below coincide with the intuition that if the $K$ clusters are
well-separated, the convergence of the estimates becomes almost equivalent to having
$K$ separate linear regression models.  
\subsection{MLR under Gaussian noise}
Consider a model where the 
data $(\bx_i, y_i) \in \mathbb{R}^{3}$ are generated independently by 
\begin{align*} 
x_{i, 1} \sim & UNIFORM(0, 1), \quad x_{i, 2} = 1, \\
 \bbeta_1 = & (-0.93, 0.1), \quad \bbeta_2 = (0,0),\\ 
\epsilon_{i} \sim & N(0, 1),\\
y_i = & \bx_i' \bbeta_k + 0.01 \epsilon_{i}, \text{ for some } k \in [2].
\end{align*}
The first element of of $\bx_i$ consists of random samples from the uniform
distribution over $[0, 1)$.  The second element of $\bx_i$ is constant.
  $\epsilon_{i}$ are drawn from the standard normal distribution. One-half of the
  data samples are from Cluster 1. Fig.~\ref{figure:MIP_Gaussian} plots the evolution
  of $\|\bbeta^n_k - \bbeta_{k}\|$ as a function of the number of samples $n$ used in
  solving problem (\ref{eq:MIP}), where we used $p=1$, $M=10$, did not include a
  regularization constraint, and set the time limit for each model in GUROBI to 2
  hours.
\begin{figure}[ht]
	\begin{centering}
	\includegraphics[width=\columnwidth]{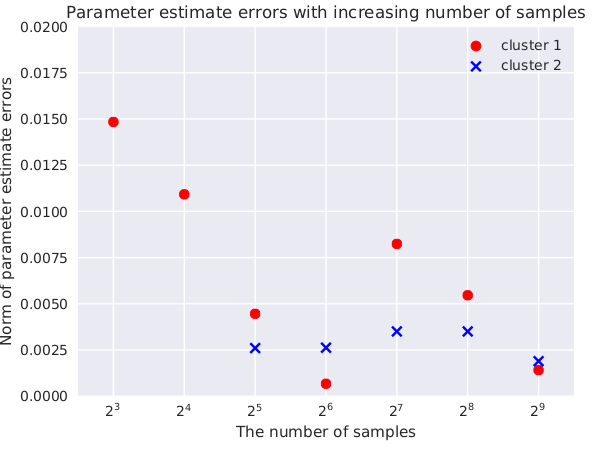}
		\caption{Gaussian noise case.}
		\label{figure:MIP_Gaussian}
	\end{centering}
\end{figure}

\subsection{MLR under uniform noise}
Consider now a model where the data samples $(\bx_i, y_i) \in \mathbb{R}^{3}$
are generated independently by
\begin{align*}
x_{i, 1} \sim & UNIFORM(0, 1), \quad x_{i, 2} = 1, \\
 \bbeta_1 = & ( -1.61, 1.25 ), \quad  \bbeta_2 = (0,0),\\ 
\epsilon_{i} \sim & UNIFORM(-1, 1),\\
y_i =& \bx_i' \bbeta_k + 0.01 \epsilon_{i}, \text{ for some } k \in [2].
\end{align*}
The first element of $\bx_i$ consists of random samples from the uniform distribution
over $[0, 1)$, and the second element is the constant $1$.  $\epsilon_{i}$ are drawn
  from the uniform distribution over $[-1, 1)$. One-half of the samples are generated
    from Cluster 1. Fig. \ref{figure:MIP_Uniform} plots $\|\bbeta^n_k - \bbeta_{k}
    \|$ as a function of the number of samples used in (\ref{eq:MIP}).  
\begin{figure}[ht]
	\begin{centering}
\includegraphics[width=\columnwidth]{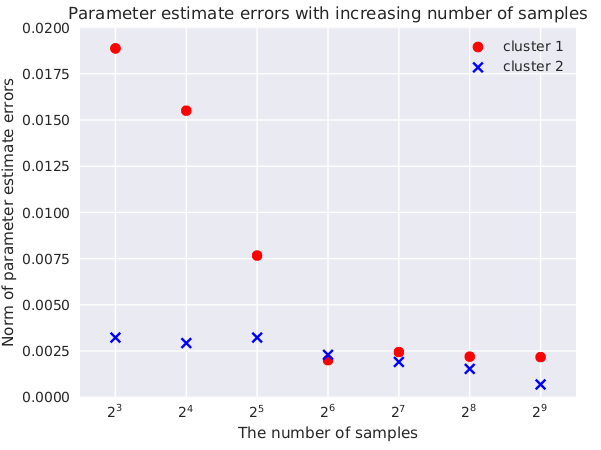}
		\caption{Uniform noise case.}
		\label{figure:MIP_Uniform}
	\end{centering}
\end{figure}

\section{Conclusions} \label{sec:con}

We established the convergence of parameter estimates for regularized mixed linear
regression models with multiple components in the noiseless case.  Regularized linear
regression can be seen as a specific case of MLR with a single cluster.  We establish
strong consistency and characterize the convergence rate of parameter estimates for
such regression models subject to martingale difference noise under very weak
assumptions on the data distribution.

To the best of our knowledge, our paper is the first to study strong consistency of
parameter estimates for mixed linear regression models under general noise conditions
and general feature conditions rather than convergence with high probability.  It can
be used directly and extended in many areas, including but not limited to system
identification and control, econometric theory and time series analysis.

\bibliographystyle{IEEEtran}
\bibliography{MIP_LS_wty-final}

\end{document}

%% file: mymacros_2cln.tex
\newcommand{\be}[1]{\begin{equation}\label{#1}}
\newcommand{\benon}{\begin{equation*}}  
\newcommand{\bemuln}[1]{\begin{multline}\label{#1}}
\newcommand{\bemul}{\begin{multline*}}
\newcommand{\bee}{\begin{eqnarray*}}
\newcommand{\eee}{\end{eqnarray*}}
\newcommand{\been}[1]{\begin{eqnarray}\label{#1}}
\newcommand{\eeen}{\end{eqnarray}}
\newcommand{\began}[1]{\begin{gather}\label{#1}}
\newcommand{\bega}{\begin{gather*}}
\newcommand{\bealn}[1]{\begin{align}\label{#1}}
\newcommand{\beal}{\begin{align*}}
\newcommand{\bealatn}[2]{\begin{alignat}{#1}\label{#2}}
\newcommand{\bealat}{\begin{alignat*}}
\newcommand{\bexalatn}[1]{\begin{xalignat}\label{#1}}
\newcommand{\bexalat}{\begin{xalignat*}}

{\theoremstyle{plain} \newtheorem{thm}{Theorem}[section]

\newtheorem{lemma}[thm]{Lemma}

}
{\theoremstyle{break} \theorembodyfont{\it}

 }


%% file: lpsymbols.tex
\def\bx{{\mathbf x}}

\def\texitem#1{\par\smallskip\noindent\hangindent 25pt
               \hbox to 25pt {\hss #1 ~}\ignorespaces}

\newcommand{\scrS}{\mathcal{S}}

\newcommand{\bbeta}{\boldsymbol{\beta}}

\def\P{\mathsf{P}}
\def\E{\mathsf{E}}

%% file: cluster-regression-cdc19-final.bbl
\begin{thebibliography}{10}
\providecommand{\url}[1]{#1}
\csname url@samestyle\endcsname
\providecommand{\newblock}{\relax}
\providecommand{\bibinfo}[2]{#2}
\providecommand{\BIBentrySTDinterwordspacing}{\spaceskip=0pt\relax}
\providecommand{\BIBentryALTinterwordstretchfactor}{4}
\providecommand{\BIBentryALTinterwordspacing}{\spaceskip=\fontdimen2\font plus
\BIBentryALTinterwordstretchfactor\fontdimen3\font minus
  \fontdimen4\font\relax}
\providecommand{\BIBforeignlanguage}[2]{{%
\expandafter\ifx\csname l@#1\endcsname\relax
\typeout{** WARNING: IEEEtran.bst: No hyphenation pattern has been}%
\typeout{** loaded for the language `#1'. Using the pattern for}%
\typeout{** the default language instead.}%
\else
\language=\csname l@#1\endcsname
\fi
#2}}
\providecommand{\BIBdecl}{\relax}
\BIBdecl

\bibitem{yi2014alternating}
X.~Yi, C.~Caramanis, and S.~Sanghavi, ``Alternating minimization for mixed
  linear regression,'' in \emph{International Conference on Machine Learning},
  2014, pp. 613--621.

\bibitem{zhong2016mixed}
K.~Zhong, P.~Jain, and I.~S. Dhillon, ``Mixed linear regression with multiple
  components,'' in \emph{Advances in neural information processing systems},
  2016, pp. 2190--2198.

\bibitem{chaganty2013spectral}
A.~T. Chaganty and P.~Liang, ``Spectral experts for estimating mixtures of
  linear regressions,'' in \emph{International Conference on Machine Learning},
  2013, pp. 1040--1048.

\bibitem{park2017algorithms}
Y.~W. Park, Y.~Jiang, D.~Klabjan, and L.~Williams, ``Algorithms for generalized
  clusterwise linear regression,'' \emph{INFORMS Journal on Computing},
  vol.~29, no.~2, pp. 301--317, 2017.

\bibitem{paoletti2007identification}
S.~Paoletti, A.~L. Juloski, G.~Ferrari-Trecate, and R.~Vidal, ``Identification
  of hybrid systems a tutorial,'' \emph{European journal of control}, vol.~13,
  no. 2-3, pp. 242--260, 2007.

\bibitem{vidal2008recursive}
R.~Vidal, ``Recursive identification of switched arx systems,''
  \emph{Automatica}, vol.~44, no.~9, pp. 2274--2287, 2008.

\bibitem{yi2015regularized}
X.~Yi and C.~Caramanis, ``Regularized em algorithms: A unified framework and
  statistical guarantees,'' in \emph{Advances in Neural Information Processing
  Systems}, 2015, pp. 1567--1575.

\bibitem{balakrishnan2017statistical}
S.~Balakrishnan, M.~J. Wainwright, B.~Yu \emph{et~al.}, ``Statistical
  guarantees for the em algorithm: From population to sample-based analysis,''
  \emph{The Annals of Statistics}, vol.~45, no.~1, pp. 77--120, 2017.

\bibitem{yi2016solving}
X.~Yi, C.~Caramanis, and S.~Sanghavi, ``Solving a mixture of many random linear
  equations by tensor decomposition and alternating minimization,'' \emph{arXiv
  preprint arXiv:1608.05749}, 2016.

\bibitem{hand2018convex}
P.~Hand and B.~Joshi, ``A convex program for mixed linear regression with a
  recovery guarantee for well-separated data,'' \emph{Information and
  Inference: A Journal of the IMA}, vol.~7, no.~3, pp. 563--579, 2018.

\bibitem{li2018learning}
Y.~Li and Y.~Liang, ``Learning mixtures of linear regressions with nearly
  optimal complexity,'' \emph{arXiv preprint arXiv:1802.07895}, 2018.

\bibitem{chen2014convex}
Y.~Chen, X.~Yi, and C.~Caramanis, ``A convex formulation for mixed regression
  with two components: Minimax optimal rates,'' in \emph{Conference on Learning
  Theory}, 2014, pp. 560--604.

\bibitem{yen2018mixlasso}
I.~E.-H. Yen, W.-C. Lee, K.~Zhong, S.-E. Chang, P.~K. Ravikumar, and S.-D. Lin,
  ``Mixlasso: Generalized mixed regression via convex atomic-norm
  regularization,'' in \emph{Advances in Neural Information Processing
  Systems}, 2018, pp. 10\,891--10\,899.

\bibitem{yin2018learning}
D.~Yin, R.~Pedarsani, Y.~Chen, and K.~Ramchandran, ``Learning mixtures of
  sparse linear regressions using sparse graph codes,'' \emph{IEEE Transactions
  on Information Theory}, 2018.

\bibitem{bertsimas2015or}
D.~Bertsimas and A.~King, ``Or forum—an algorithmic approach to linear
  regression,'' \emph{Operations Research}, vol.~64, no.~1, pp. 2--16, 2015.

\bibitem{bertsimas2016best}
D.~Bertsimas, A.~King, R.~Mazumder \emph{et~al.}, ``Best subset selection via a
  modern optimization lens,'' \emph{The annals of statistics}, vol.~44, no.~2,
  pp. 813--852, 2016.

\bibitem{xu2016joint}
T.~Xu, T.~S. Brisimi, T.~Wang, W.~Dai, and I.~C. Paschalidis, ``A joint sparse
  clustering and classification approach with applications to hospitalization
  prediction,'' in \emph{Decision and Control (CDC), 2016 IEEE 55th Conference
  on}, 2016, pp. 4566--4571.

\bibitem{ACC-2018-PIEEE}
T.~S. Brisimi, T.~Xu, T.~Wang, W.~Dai, W.~G. Adams, and I.~C. Paschalidis,
  ``Predicting chronic disease hospitalizations from electronic health records:
  An interpretable classification approach,'' \emph{Proceedings of the IEEE},
  vol. 106, no.~4, pp. 690--707, 2018.

\bibitem{ACC-diab-bri-xu-2019}
T.~S. Brisimi, T.~Xu, T.~Wang, W.~Dai, and I.~C. Paschalidis, ``Predicting
  diabetes-related hospitalizations based on electronic health records,''
  \emph{Statistical Methods in Medical Research}, 2019.

\bibitem{bertsimas2007classification}
D.~Bertsimas and R.~Shioda, ``Classification and regression via integer
  optimization,'' \emph{Operations Research}, vol.~55, no.~2, pp. 252--271,
  2007.

\bibitem{Tib96}
R.~Tibshirani, ``Regression shrinkage and selection via the {LASSO},''
  \emph{Journal of the Royal Statistical Society. Series B (Methodological)},
  vol.~58, no.~1, pp. 267--288, 1996.

\bibitem{chatterjee2013assumptionless}
S.~Chatterjee, ``Assumptionless consistency of the lasso,'' \emph{arXiv
  preprint arXiv:1303.5817}, 2013.

\bibitem{chen2018robust}
\BIBentryALTinterwordspacing
R.~Chen and I.~C. Paschalidis, ``A robust learning approach for regression
  models based on distributionally robust optimization,'' \emph{Journal of
  Machine Learning Research}, vol.~19, no.~13, 2018. [Online]. Available:
  \url{http://jmlr.org/papers/v19/17-295.html}
\BIBentrySTDinterwordspacing

\bibitem{lai1982least}
T.~L. Lai, C.~Z. Wei \emph{et~al.}, ``Least squares estimates in stochastic
  regression models with applications to identification and control of dynamic
  systems,'' \emph{The Annals of Statistics}, vol.~10, no.~1, pp. 154--166,
  1982.

\bibitem{chow2012probability}
Y.~S. Chow and H.~Teicher, \emph{Probability theory: independence,
  interchangeability, martingales}.\hskip 1em plus 0.5em minus 0.4em\relax
  Springer Science \& Business Media, 2012.

\bibitem{chen2012identification}
H.-F. Chen and L.~Guo, \emph{Identification and stochastic adaptive
  control}.\hskip 1em plus 0.5em minus 0.4em\relax Springer Science \& Business
  Media, 2012.

\bibitem{nielsen2005strong}
B.~Nielsen, ``Strong consistency results for least squares estimators in
  general vector autoregressions with deterministic terms,'' \emph{Econometric
  Theory}, vol.~21, no.~3, pp. 534--561, 2005.

\bibitem{wei1987adaptive}
C.~Wei \emph{et~al.}, ``Adaptive prediction by least squares predictors in
  stochastic regression models with applications to time series,'' \emph{The
  Annals of Statistics}, vol.~15, no.~4, pp. 1667--1682, 1987.

\bibitem{chow1965local}
Y.~S. Chow, ``Local convergence of martingales and the law of large numbers,''
  \emph{The Annals of Mathematical Statistics}, vol.~36, no.~2, pp. 552--558,
  1965.

\end{thebibliography}
